\documentclass{article}

\PassOptionsToPackage{square,numbers}{natbib}
\usepackage[final]{nips_2016}
\usepackage{hyperref}       

\usepackage[utf8]{inputenc} 
\usepackage[T1]{fontenc}    
\usepackage{url}            
\usepackage{booktabs}       
\usepackage{nicefrac}       
\usepackage{microtype}      
\usepackage{graphicx}
\usepackage{amsmath}
\usepackage{amsthm}
\usepackage{amssymb}
\usepackage{cleveref}
\usepackage{booktabs}
\usepackage{tabularx}
\usepackage{todonotes}
\newtheorem{prop}{Lemma}
\newtheorem{thm}{Theorem}

\newcommand{\mypar}[1]{\textbf{#1}}
\renewcommand{\todo}[1]{}

\crefname{equation}{equation}{equations}
\Crefname{equation}{Equation}{Equations}
\crefrangelabelformat{equation}{(#3#1#4--#5#2#6)}

\crefmultiformat{equation}{(#2#1#3}{,#2#1#3)}{#2#1#3}{#2#1#3}
\Crefmultiformat{equation}{(#2#1#3}{,#2#1#3)}{#2#1#3}{#2#1#3}

\title{Theta-RBM: Unfactored Gated Restricted Boltzmann Machine for Rotation-Invariant Representations}

\graphicspath{{images/}}

\author{
   Mario Valerio Giuffrida\\
   PRIAn Research Unit\\
   IMT School Advanced Studies \\
   Lucca (LU), Italy\\
   \texttt{valerio.giuffrida@imtlucca.it} \\
   \And
   Sotirios A. Tsaftaris\\
   School of Engineering\\
   University of Edinburgh\\
   Edinburgh, UK\\
   \texttt{s.tsaftaris@ed.ac.uk}
}

\begin{document}

\maketitle

\begin{abstract}
Learning invariant representations is a critical task in computer vision. In this paper, we propose the Theta-Restricted Boltzmann Machine ($\theta$-RBM in short), which builds upon the original RBM formulation and injects the notion of rotation-invariance during the learning procedure. In contrast to previous approaches, we do not transform the training set with all  possible rotations. Instead, we rotate the gradient filters when they are computed during the \textit{Contrastive Divergence} algorithm. We formulate our model as an unfactored gated Boltzmann machine, where another input layer is used to modulate the input visible layer to drive the optimisation procedure. Among our contributions is a mathematical proof that demonstrates that $\theta$-RBM is able to learn rotation-invariant features according to a recently proposed invariance measure. Our method reaches an invariance score of $\sim 90\%$ on \textit{mnist-rot} dataset, which is the highest result compared with the baseline methods and the current state of the art in transformation-invariant feature learning in RBM. Using an SVM classifier, we also showed that our network learns discriminative features as well, obtaining $\sim 10\%$ of testing error.
\end{abstract}

\section{Introduction}
\label{sec:intro}

Most of the applications in computer vision require suitable image representations, which are invariant to certain geometrical transformations. 
Recently, learned image representations have demonstrated impressive performance and have become a powerful tool in computer vision. In particular, neural networks have been shown to learn more discriminative image representations, by learning task-specific filters to apply to input images \cite{Coates2011}. A well-known neural network that learns features in an unsupervised fashion is the Restricted Boltzmann Machine (RBM) \cite{Smolensky1986}, characterised by a bipartite graph whose sides are referred as \textit{visible} and \textit{hidden} units respectively. 
In its standard formulation RBM does not accommodate nuisance factors in the scene, hence it cannot learn invariant features.

The easiest way to achieve a notion of invariance is to present to the learning algorithm as much as variability as possible. When this is not possible, simulated variability is introduced by applying transformations artificially to data, in what is known as \textit{dataset augmentation}. The act of altering input data has been exploited thoroughly in this regard (e.g. \cite{Schmidt2012,Sohn2012}). However, this has disadvantages such as: (i) introducing alias due to pixel interpolation in the transformed images; and (ii) the transformed data may not span over the entire space of transformations. Instead, recently in \cite{Cheng2013}, gradient information is exploited to transform data into a common reference frame and thus eschew augmentation.

In this paper we present the $\theta$-RBM that learns image representations that are invariant to rotations. Taking inspiration by previous works  \cite{Cheng2013,Sohn2013}, we propose an unfactored gated Restricted Boltzmann Machine that optimises a third-order tensor, which learns rotation-invariant features from input images. In our network, shown in \Cref{fig:theta-rbm}, we use the visible layer $x$ to provide images, while the input layer $r$ is \textit{only} used to indicate the dominant rotation of input patches. An input image is paired with the corresponding rotation, which if found by computing the dominant gradient of this image. \\

\textbf{Contributions:} Our contributions are: (i) a compact gated formulation of RBMs for rotation invariance; (ii) which  mathematically prove that it learns rotation-invariant features, on the basis of an known invariance measure \cite{Rasmus2015}; (ii) we do not require dataset augmentation or pre-training to learn rotated patterns as required by others \cite{Sohn2013}, thus training is faster; (iii) the \textit{unfactored} third-order tensor uses a limited amount of slices leading to a more compact data representation compared with other methods that use several factors \cite{Memisevic2007}; and finally (iv) we also show empirically that our method learns rotation-invariant features even when not presented with rotated data (ie. using an unrotated training set).

The remainder of this paper is organised as follows. \Cref{sec:rel_works} reviews related work. \Cref{sec:model} presents the proposed unfactored gated formulation of $\theta$-RBM. In \Cref{sec:inv} we provide the mathematical proof of invariant feature learning. In \Cref{sec:results} we report experimental results of our method on \textit{mnist-rot}. \Cref{sec:conclusions} offers conclusions.

\section{Related Works}
\label{sec:rel_works}

\begin{figure}[t]
	\centering
	\includegraphics[width=0.57\textwidth]{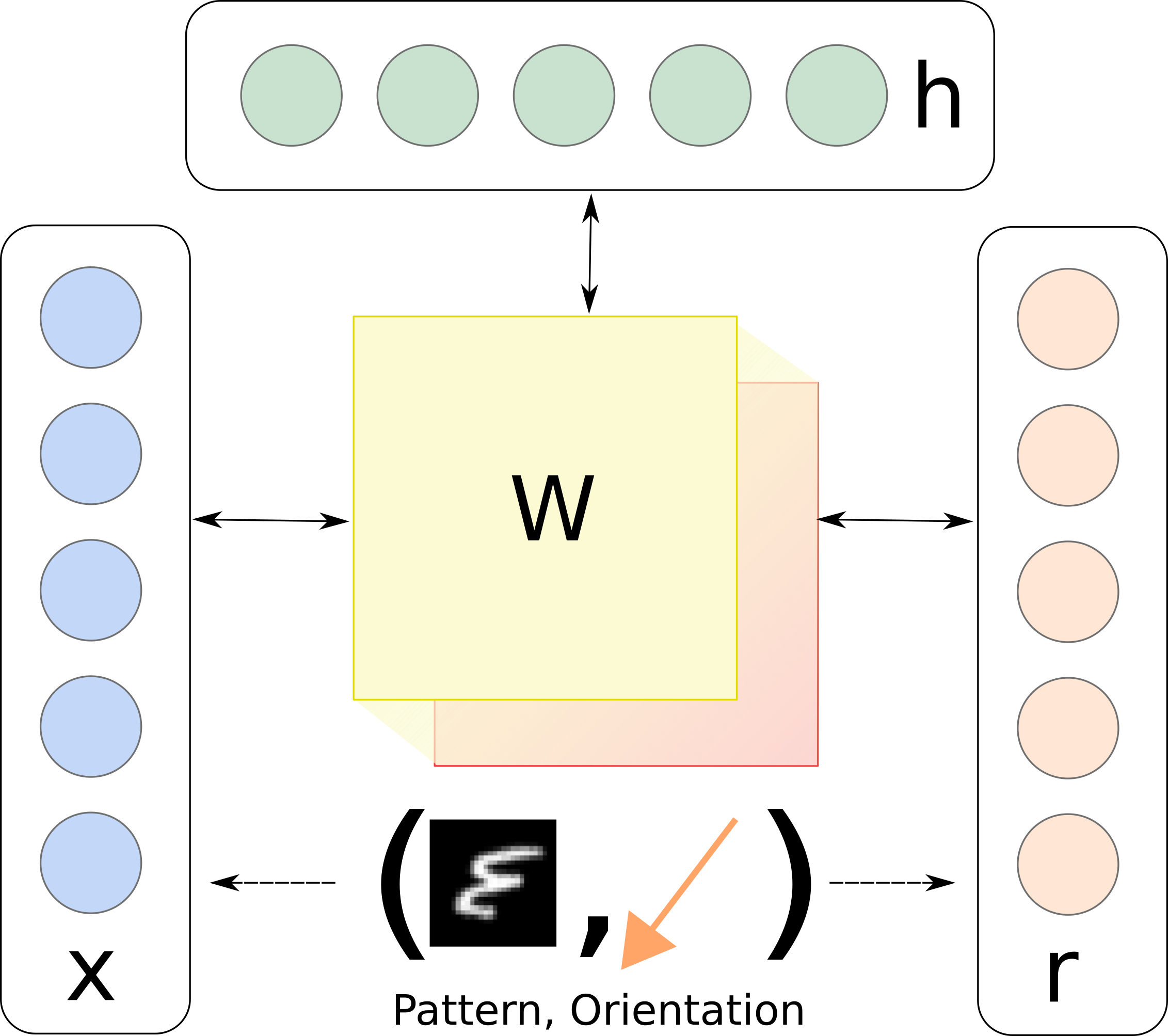}
	\caption{Representation of our $\theta$-RBM as an unfactored gated Boltzmann machine.  $x$ and $r$ are visible layers respectively representing input image patterns and rotations and $h$ is the hidden layer that learns the relationship among inputs via the third-order $\mathbf{W}$ shown in the middle.}
	\label{fig:theta-rbm}
\end{figure}

\mypar{Invariance in RBM:} Mainly due to their wide utilisation in deep learning, several adaptations and extensions of the original RBM model have been proposed to accommodate transformed variants of the same input. A Deep Belief Network (DBN) \cite{Hinton2006}, obtained by stacking RBMs, can learn higher-level representations with invariance characteristics \cite{Goodfellow2009}. In \cite{Shou2013} the authors extend the DBN and train it several times, such that at each iteration the learned weights are transformed w.r.t to a set $\mathcal{S}$ of transformations. Such approach increases the computational burden of training a single time a DBN linearly with the cardinality of $\mathcal{S}$, which may require a considerable amount of time. The Transformation-Invariant RBM (TI-RBM) \cite{Sohn2012} can learn features that are invariant to local transformations, by altering the training w.r.t. a set of suitable transformations (e.g., rotations, translations, scaling). The main drawbacks of this methodology is very similar to the straightforward approach of dataset augmentation, since transformed images may not fully cover the entire manifold where the data may lie. In \cite{Cheng2013}, patches are transformed beforehand using SIFT features, applying rotation and scaling to normalise the inputs. In \cite{Jaderberg2015}, the authors present the Spatial Transformer, a differentiable layer incorporated inside a convolutional architecture to learn affine transformations from the feature maps. Estimating the transformations that apply to an input image showed to be an interesting approach towards discriminative invariant features, and it is the direction we pursued and incorporated on our $\theta$-RBM.

\mypar{Gated networks:} Recently, a radically different view of network architectures has been proposed, known as gated networks \cite{Memisevic2013} which have demonstrated impressive results on learning image pairs. The basic idea of gated networks is to optimise the parameters involved in the model, such that the input layer $x$ is \textit{modulated} by another visible layer $y$ to find a common representation $h$. In \cite{Susskind2011}, the authors show that a Gated Boltzmann Machine (GBM) can learn rotated versions of faces. Even though the discriminative power of GBMs has been demonstrated in several applications \cite{Karianakis2013,Memisevic2013,Memisevic2009,Sohn2013,Susskind2011}, they do need a pair of related images to be trained. Nonetheless, in \cite{Sohn2013}, the authors showed that GBM can be used for a different purpose by just slightly varying the original formulation. In fact, instead using the $y$ layer to provide another image, they use it as feature selection activation layer, such that the model learns which features are more discriminative for the task at hand. These GBMs are able to efficiently use the three-way interaction of two input layers and a hidden layer which in turn provides a common representation for the inputs. Furthermore, in \cite{Mocanu2015} the authors showed that GBM can be extended such that they can deal with four-way interaction, adding a further layer to provide class labels. The two visible layers appearing in a three-way GBM are used during the optimisation to find correlation between them. As a result, the model is able to learn features that are invariant to affine transformations, when pairs of related data are provided (e.g., pairs of faces \cite{Susskind2011}). Following a similar approach of \cite{Sohn2013}, we use the third layer as input, in order to encode the magnitude of rotation of each input image. Therefore, the learning procedure is strongly conditioned by this interaction, and we rotate the learned filters such that the network learns rotation-invariant features.

\mypar{Invariance measures:} in \cite{Goodfellow2009} a measure invariance is introduced, which computes the mean activation of features extracted from transformed images over the mean activation of the features obtained from the original dataset. Even though this metric provides  a score of invariance, it is unbounded and therefore it is not easy to understand when a network is reaching the maximum invariance. Another metric that has been utilised in literature is the mean squared error over the $L_2$ distance of normalised features \cite{Zou2012}. While being bounded its main drawback is that vectors need to be  normalised prior to computing the invariance score. In fact, magnitude difference, resulting from transformed images, could be due to lack of invariance, which may bias a classification algorithm. In fact in \cite{Dosovitskiy2014}, the authors propose to compute the average over the classification scores of the transformed images. The latter approach requires two steps, namely feature learning and classifier training, which increase the whole complexity of the process. Another metric of invariance was proposed in \cite{Rasmus2015}, which behaves similar to the autocorrelation. It is bounded in $[0,1]$ with `1' reflecting full invariance. Amongst the proposed metrics in literature, we chose to develop our theory upon the latter one, because of its defined range and also its computational efficiency when used in empirical demonstrations.
\vspace{-0.8em}

\section{$\theta$-Restricted Boltzmann Machine}
\label{sec:model}

We present an unfactored gated Boltzmann machine (GBM) for learning rotation-invariant features. The original formulation of unfactored GBM correlates the interaction between the two visible layers and the hidden layer with a third-order tensor $\mathbf{W} \in \mathbb{R}^{H\times V\times S}$  \cite{Memisevic2009}, where $H$ is the number of hidden units of the layer $h$, $V$ is the dimension of the visible layer $x$, and $S$ is the dimension of the other visible (rotation) layer $r$, shown pictorially in \Cref{fig:theta-rbm}. We build upon this model to propose our $\theta$-RBM. For sake of presentation clarity, the following definitions and analyses are based upon the standard Bernoulli RBM, but the model can be easily extended to treat real-values as well, as described in \cite{Cho2011}. In fact, all our experiments are based on a real-valued Gaussian extension.

Firstly, let us define a support set of rotations $\mathcal{S} = \left\lbrace \varphi_1, \varphi_2, \ldots, \varphi_S \right\rbrace$, containing $S$ rotations. Then, the energy function that characterises our model is
\begin{equation}
\label{eq:theta_rbm_e}
E(v,h,r) = \sum_{s=1}^{S} \sum_{j=1}^{H} \sum_{k=1}^{V} r_s \left( - v_k h_j W_{jks}  - b_j h_j - c_k v_k \right),
\end{equation}
where $\mathbf{v}$, $\mathbf{h}$, and $\mathbf{r}$ are random vectors that can take only binary values. The third-order tensor $\mathbf{W}$ defines the relationship with the layers in the model, $c$ and $b$ define the bias terms for the visible $x$, and hidden layer $h$ respectively. A further constraint we add is the sparseness on $\mathbf{r}$
\begin{equation}
\label{eq:rot_constr}
\sum_{s=1}^S r_s  = 1.
\end{equation}
The vector $\mathbf{r}$ is used as one-hot indicator function, such that if a patch is rotated by $\varphi_s$ degree, then $r_s = 1$ and the remaining $r_k = 0$, $\forall k\neq s$. Because $\mathcal{S}$ defines the set of rotations that our model can learn, an input pattern $x$ can be subjected to one rotation. It is easy to check that the conditional probabilities deriving from \Cref{eq:theta_rbm_e} are
\begin{equation}
	\label{eq:h_given_x}
	p(h_j = 1 |\mathbf{x},\mathbf{r}) = \sigma\left( \sum_{s=1}^{S} r_s \left(b_j + \mathbf{W}_{j,\bullet,s} \mathbf{x}\right)\right),
\end{equation}
\begin{equation}
\label{eq:x_given_h}
p(x_k=1|\mathbf{h},\mathbf{r}) = \sigma\left( \sum_{s=1}^{S} r_s \left(c_k + \mathbf{h}' \mathbf{W}_{\bullet,k,s}\right)\right),
\end{equation}
where $\mathbf{W}_{j,\bullet,s}$ denotes the $j$-th row in the slice $s$ of the third-order tensor $\mathbf{W}$, $\mathbf{W}_{\bullet,k,s}$ denotes the $k$-th column in the $s$-th slice in $\mathbf{W}$, $\mathbf{h}'$ is the transpose of the column vector $\mathbf{h}$, and $\sigma(y)$ is the sigmoid function. In our model, the layer $r$ is treated as input and rotation is found as the dominant orientation experimental computed using the angle with highest frequency in the histogram of oriented gradients.


For a single pattern taken from the dataset $(\tilde{x},\tilde{s})$, image patch and orientation respectively, the Contrastive Divergence algorithm \cite{Hinton2002} is performed as usual between the visible layer $x$ and the hidden layer $h$. This procedure computes the gradient $\mathbf{\nabla W_{\tilde{s}}}$, that is the $\tilde{s}$-th frontal slice of the tensor $\mathbf{W}$ (we denoted $\mathbf{W_{\bullet,\bullet,s}}$ as $\mathbf{W_s}$ for brevity). The remaining slices in $\mathbf{W}$ are computed as follows
\begin{equation}
	\label{eq:shared_term}
	\nabla \mathbf{W}_{k} = R_\theta (\nabla \mathbf{W}_{\tilde{s}}),\; \forall k=1,2,\ldots,S, k\neq \tilde{s}.
\end{equation}
\vspace{-1em}
\begin{equation}
\label{eq:rot}
\textrm{where}\qquad R_\theta(\mathbf{A}) = \mathbf{A}\mathbf{T}'_\theta,
\end{equation}
such that $R_\theta(\mathbf{A})$ is a transformation function that rotates by $\theta = \varphi_t - \varphi_{\tilde{s}}$ degrees each row in $\mathbf{A}$. The matrix $\mathbf{T}\in \mathbb{R}^{n\times n}$ defines the rotation and is represented as a particular permutation matrix such that $\mathbf{T} \mathbf{T}' = \mathbf{T}'\mathbf{T} = \mathbf{I}$, $\operatorname{det}(\mathbf{T})=\pm 1$. The input matrix $\mathbf{A}$ is transposed in $R_\theta(\bullet)$ because filters in a slice in $\mathbf{W}$ are disposed row-wise. Specifically, a filter is a single row in any slice of the tensor $\mathbf{W}$ and it has the same dimension of input images. 
Once the third-order tensor $\mathbf{\nabla W}$ is computed, the weight matrix is updated as follows
\begin{equation}
\label{eq:update}
\mathbf{W^{(t)}} = \mathbf{W^{(t-1)}} + \eta \mathbf{\nabla W^{(t)}} + \alpha^{(t)} \mathbf{\nabla W^{(t-1)}},
\end{equation}
where $\eta$ is the learning rate, $\alpha^{(t)}$ is the momentum at time $t$. In \eqref{eq:update} we use the superscript index $t$ to indicate the iteration time. 


\section{Proving Rotation Invariance}
\label{sec:inv}
We want to demonstrate mathematically that the method above learns rotation invariant features. We used the invariance measure proposed in \cite{Rasmus2015}. Henceforth, we will refer to this measure as $\gamma$\textit{-score} and we will provide a general definition, then we will develop our theory upon it. We adopted this measure because it ranges between `0' and `1', which indicate full variance and full invariance to transformations respectively (see also \Cref{sec:rel_works}). 

\subsection{The $\gamma$-score}



Let $\mathcal{S}$ be a set of any transformations\footnote{We slightly abuse the definition of the set here to be any transform just for the purpose of presenting the score but our analysis and results presume that $\mathcal{S}$ contains only rotations.} that can be applied to the dataset $\mathcal{X}$. Moreover, let $h_j(\mathbf{x})$ be the state of the $j$-th hidden neuron when the pattern $x$ is provided (e.g., similar to \Cref{eq:h_given_x}). The mean activation of a given element in the training set $\mathcal{X}$ across all the transformations in $\mathcal{S}$ is computed as follows:
\begin{equation}
\label{eq:mu}
\mu_j(\mathbf{x}) = \frac{1}{S} \sum_{s \in \mathcal{S}} h_j (\mathbf{T}_s(\mathbf{x})),
\end{equation}
where $S$ is the cardinality of the support set $\mathcal{S}$. Hence, the $\gamma$-score for the $j$-th hidden unit is given by
\begin{equation}
\label{eq:inv_meas}
\gamma_j = \frac{\textrm{var}\left\lbrace \mu_j(\mathbf{x}) \right\rbrace_{x\in \mathcal{X}}}{\textrm{var}\left\lbrace h_j(\mathbf{x})\right\rbrace_{x\in \mathcal{X}}},
\end{equation}
that is the variance of the mean activation of all transformed samples in $\mathcal{X}$ over the variance of activations of the original data $\mathcal{X}$. The invariance measure in \Cref{eq:inv_meas} achieves its maximum value `1' when features are invariant to the  set of transformations $\mathcal{S}$.

\subsection{Theorem statement}

In this section, we will prove that our $\theta$-RBM learns rotation-invariant features, when the learning procedure described in \Cref{sec:model} is used. Since the maximum value of the $\gamma$-score is achieved only when numerator and denominator are equal, we will show that the argument of the variance functions $\mu_j$ and $h_j$ are equal. For our purposes, the support set of transformations $\mathcal{S}$ is formed by rotations. Before to proceed with the proof, we need the support of the following lemma.

\begin{prop}
	\label{prop:prop1}
	Given a third-order tensor $\mathbf{W} \in \mathbb{R}^{H\times V \times S}$ optimised as described in \Cref{sec:model}, then $\mathbf{W}^{(t)}_{s'} = R_\theta (\mathbf{W}^{(t)}_{s})$, with $\theta = \varphi_{s'} - \varphi_s$,\; $\varphi_{s'}, \varphi_{s} \in \mathcal{S}$.
\end{prop}

\begin{proof}
	We will proceed by induction over the iteration $t$. For the base case $t=0$ \todo{better to say initial condition?}, we impose that:
	\begin{itemize}
		\item $\mathbf{M} \in R^{H\times V}$ matrix initialised somehow (e.g., normal distribution)\footnote{We observed that the base case of the induction can be relaxed. Experimental evidence showed that by initialising $\mathbf{W}$ with random numbers drawn from a normal distribution it is still possible to have rotation-invariant features.},
		\item $\mathbf{W}^{(0)}_{s} = R_{\varphi_s} (\mathbf{M})$, $\forall s \in \mathcal{S}$.
	\end{itemize}

namely all the slices in $\mathbf{W}^{(0)}$ are initialised as rotated versions of $\mathbf{M}$, which initially can be any matrix. Now, let us suppose that the lemma is true until $t-1$, and demonstrate it for $t$. 

\begin{align*}
	R_\theta(\mathbf{W}^{(t)}_{s}) & = \underbrace{R_\theta \left( \mathbf{W}^{(t-1)}_{s} + \eta \mathbf{\nabla W}^{(t)}_{s} + \alpha^{(t)} \mathbf{\nabla W}^{(t-1)}_{s} \right)}_\textrm{From \Cref{eq:update}}\\
	& = \underbrace{R_\theta \left( \mathbf{W}^{(t-1)}_{s}\right) + \eta R_\theta \left(\mathbf{\nabla W}^{(t)}_{s}\right) + \alpha^{(t)} R_\theta \left(\mathbf{\nabla W}^{(t-1)}_{s} \right)}_\textrm{Using the linearity of the transformation} \\
	& = \underbrace{\mathbf{W}^{(t-1)}_{s'}}_\textrm{By induction} + \eta \underbrace{\mathbf{\nabla W}^{(t)}_{s'} + \alpha^{(t)} \mathbf{\nabla W}^{(t-1)}_{s'}}_\textrm{From \Cref{eq:shared_term}} = \mathbf{W}^{(t)}_{s'}.
\end{align*}
\end{proof}



\begin{thm}
	\label{thm:thm1}
	Under the hypotheses of \Cref{prop:prop1} and given a support set $\mathcal{S}$ of $S$ rotations, $\gamma = 1\,$ for $\theta$-RBM, i.e., features are invariant under the rotations in $\mathcal{S}$. 
\end{thm}
	
\begin{proof}
	We have to prove that $\gamma_j = 1$, $j=1, 2, \ldots, H$. From \Cref{eq:inv_meas}, we will show that the numerator and denominator coincide. Some preliminaries that help exposition. In \Cref{eq:rot} the rotation transformation was defined as a row-wise operation. Since we used as convention that vectors are column-wise, we will use the following definition
	\begin{equation}
	\label{eq:rot_t}
	R'_\theta ( \mathbf{x} ) \equiv \left[ R_\theta (\mathbf{x}')\right]' = \left[ \mathbf{x}'\mathbf{T}_\theta' \right]',
	\end{equation}
	such that, given a column vector as argument, we obtain a column vector as result. Now, starting from the definition of $\mu_j$ showed in \Cref{eq:mu}, we get
    \begin{align*}
		\mu_j(\mathbf{x}) & = \frac{1}{S} \sum_{\theta \in \mathcal{S}} h_j(R'_\theta(\mathbf{x}))  =  \frac{1}{S} \sum_{\theta \in \mathcal{S}} \underbrace{\sigma\left( \sum_{t=1}^{S} r_t \left(b_j + \mathbf{W}_{j,\bullet,t} R'_\theta(\mathbf{x}) \right)\right)}_\textrm{From \Cref{eq:h_given_x}} \\
		& = \frac{1}{S} \sum_{\theta \in \mathcal{S}} \underbrace{\sigma\left(b_j + W_{j,\bullet,s'} R'_\theta(\mathbf{x})\right)}_\textrm{From \eqref{eq:rot_constr}. Let $r_{s'} = 1$} = \frac{1}{S} \sum_{\theta \in \mathcal{S}} \underbrace{\sigma\left(b_j + R_{\theta}(\mathbf{W}_{j,\bullet,s})  R'_{\theta}(\mathbf{x})\right)}_\textrm{From Lemma \ref{prop:prop1}, $\exists s : R_\theta(\mathbf{W}_s) = \mathbf{W}_{s'}$ } \\
		& = \frac{1}{S} \sum_{\theta \in \mathcal{S}} \underbrace{\sigma\left(b_j + \mathbf{W}_{j,\bullet,s} \mathbf{T}'_\theta \left[\mathbf{x}' \mathbf{T}'_\theta\right]' \right)}_\textrm{From Equations \eqref{eq:rot} and \eqref{eq:rot_t}}  = \frac{1}{S} \sum_{\theta \in \mathcal{S}} \underbrace{\sigma\left(b_j + \mathbf{W}_{j,\bullet,s} \mathbf{T}_\theta' \mathbf{T}_\theta \mathbf{x} \right)}_{(AB)' = B'A'} \\
		& = \frac{1}{S} \sum_{\theta \in \mathcal{S}} \underbrace{\sigma\left(b_j + \mathbf{W}_{j,\bullet,s} \mathbf{x} \right)}_{\mathbf{T}_\theta  \mathbf{T}_\theta'=\mathbf{T}'_\theta \mathbf{T}_\theta =I}  = \frac{1}{S} \sum_{\theta \in \mathcal{S}} \sigma\left(  \sum_{t=1}^S r_t \left(b_j + \mathbf{W}_{j,\bullet,t}  \mathbf{x} \right) \right) \\
		& = \frac{1}{S} \sum_{\theta \in \mathcal{S}} h_j(\mathbf{x}) = h_j(\mathbf{x}),
		\end{align*}
		\noindent where $\theta = \varphi_{s} - \varphi_{s'}$,\;\;$\varphi_{s}, \varphi_{s'} \in \mathcal{S}$. Since $\mu_j(\mathbf{x}) = h_j(\mathbf{x})$, then also their variance over all the samples in the training set is equal. Therefore $\gamma_j=1$.
\end{proof}


With this theorem, we mathematically proved that the features learnt by our $\theta$-RBM are invariant under a set of rotations $\mathcal{S}$. \todo{remind me to ask you smth}

\section{Experimental Results}
\label{sec:results}

\begin{figure}
\centering
\includegraphics[width=0.5\textwidth]{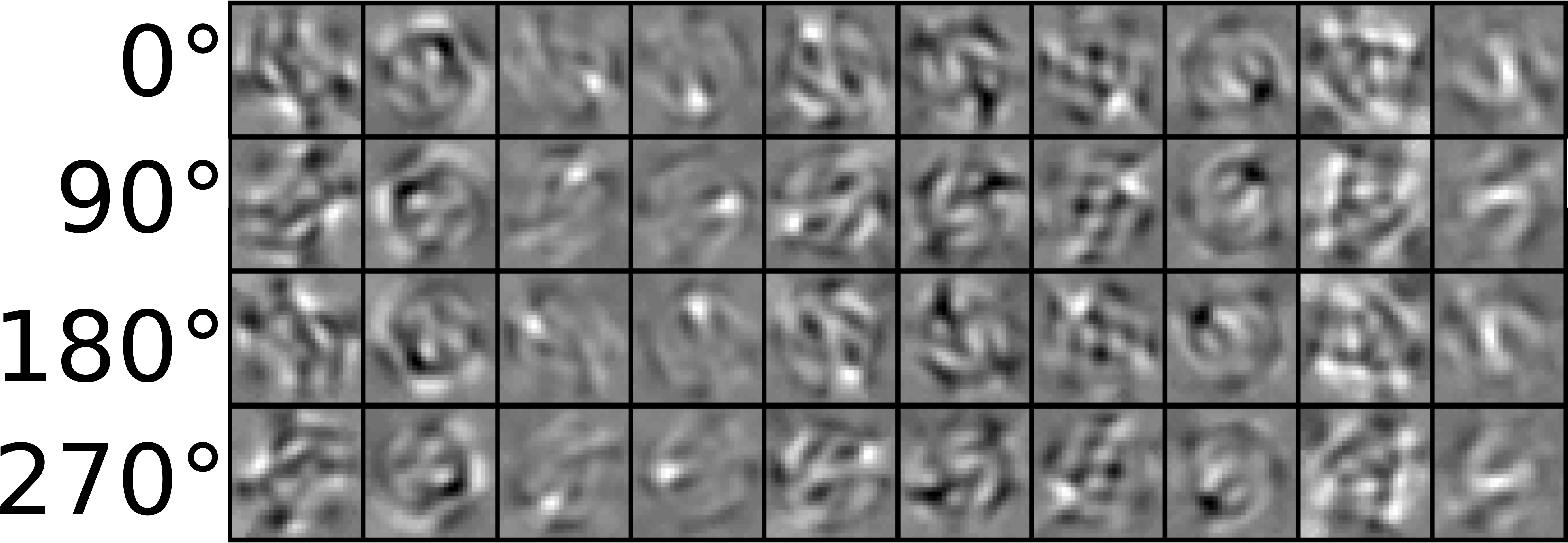}
\caption{Filters learned from $\theta$-RBM on \textit{mnist-rot} dataset. We show a subset of the learned filters for brevity. Observing the columns, filters are rotated versions of each other.}
\label{fig:filters}
\end{figure}

In this section we show experimentally that $\theta$-RBM is able to learn rotation-invariant features. Briefly summarized, training with \textit{mnist-rot} dataset \cite{Larochelle2007} we will show the claim of \Cref{thm:thm1} and we will estimate invariance according to \Cref{eq:inv_meas}. As baselines, we compared with the standard (Gaussian-Bernoulli) RBM\todo{right? [vg] si}, the \textit{Oriented RBM} (O-RBM), which is trained by pre-aligning the input images by their respectively dominant orientation. We compare also with TI-RBM --a recent state-of-the-art method \cite{Sohn2012}. We train our $\theta$-RBM using the Gaussian-Bernoulli formulation in \cite{Cho2011}, which can be straightforwardly done by adapting Equations \eqref{eq:theta_rbm_e}  and \cref{eq:h_given_x,eq:x_given_h} such that they can deal with continuous variables. To demonstrate that $\theta$-RBM does not need the input data to span the space of transformations in another experiment we train the proposed model with \textit{mnist} dataset and test it with \textit{mnist-rot}. Since our method relies on an estimation of the dominant orientation of an input image, we perturbed the estimate of orientation and evaluate classification accuracy to test robustness.

\mypar{Data normalisation:} Before training the networks, we normalised the data such that they have zero mean and unit variance for each data component (i.e., the same input feature --pixel location). In case of $\theta$-RBM, normalisation is performed within data having the same orientation.

\mypar{Discriminative power of learned features:} By these experiment we want to demonstrate that our method is able to learn rotation-invariant features, which have a strong discriminative power. We trained the networks using the following parameters: $500$ of hidden units, learning rate $\eta = 0.01$, constant momentum $\alpha = 0.9$ , and $S=9$ rotations in the support set $\mathcal{S} = \left\lbrace 0^\circ, 40^\circ, 80^\circ, 120^\circ, 160^\circ, 200^\circ, 240^\circ, 280^\circ, 320^\circ \right\rbrace$ (where applicable). We trained an RBF-SVM classifier \cite{Vapnik95}, where parameters were found via grid search. Best results were obtained with loss parameter $C=10$ and parameter of the Gaussian kernel $\sigma = 0.02$. In \Cref{tab:svm} we show the testing error on the \textit{mnist-rot} dataset. Observe that our method has the lowest error. In \Cref{fig:filters} we show a subset of filters that $\theta$-RBM learned on the \textit{mnist-rot} dataset. Observe that filters are rotated versions of each other, giving experimental evidence on the validity of \Cref{prop:prop1}.

\mypar{Demonstrating invariance:} By these experiments, we want to assess the degree of invariance that $\theta$-RBM can achieve. We use \textit{mnist-rot} for the empirical demonstration of \Cref{thm:thm1}, using the same experimental setup as above. In order to compute the $\gamma$-score (c.f. \Cref{eq:inv_meas}), we need a set of transformations to compute the $\mu_j$ (c.f. \Cref{eq:mu}). Instead of using the same support set used at training time, we instead used $\mathcal{S}' = \left\lbrace \varphi_s + \Delta : \varphi_s \in \mathcal{S} \right\rbrace$ to compute $\mu_j$, that is a set of rotations that are derived from $\mathcal{S}$, but are shifted by $\Delta$. We do that because it would result on extremely high invariance score, since the set of rotations is the one that the algorithms were trained with. In our experiments, we set $\Delta = 20^\circ$. \Cref{tab:inv_meas} we show the $\gamma$-score for our method, baseline RBM and O-RBM, and TI-RBM, applying different sparsity target. Overall our method achieves the highest $\gamma$-score, compared with the other methods. Different sparsity values appears not to influence our results, which is in stark contrast to what happens for TI-RBM, which has $-5\%$ loss when the sparsity target decreases. With the achieve $\gamma$-score of $\sim 0.9$ by our method even empirically we can see the validity of \Cref{thm:thm1}.


\begin{table}[t]
\centering
\caption{Classification of \textit{mnist-rot} digits, using RBF-SVM classifier. We report the testing error, obtained for each method, using different sparsity values.}
\label{tab:svm}
\begin{tabular}{@{}lccc@{}}
\toprule
\multicolumn{1}{c}{\textit{\textbf{}}} & Sparsity 0.3     & Sparsity 0.2     & Sparsity 0.1     \\ \midrule
RBM                                    &        15.41\%          &      16.04\%            &       16.39\%           \\
O-RBM                                  & 15.61\%          & 15.87\%          & 16.61\%          \\
TI-RBM                                 & 14.07\%          & 13.35\%          & 14.08\%          \\
\textbf{$\mathbf{\theta}$-RBM}                          & \textbf{10.31\%} & \textbf{10.08\%} & \textbf{10.85\%} \\ \bottomrule
\end{tabular}
\end{table}

\begin{table}[t]
	\centering
	\caption{The $\gamma$-score, as described in \Cref{sec:inv}, on features learned by different approaches. All the trained models used Gaussian-Bernoulli formulation to treat continuous-value data.}
	
	\label{tab:inv_meas}
	\begin{tabular}{@{}lllllll@{}}
		\toprule
		& \multicolumn{2}{c}{Sparsity 0.3}                     & \multicolumn{2}{c}{Sparsity 0.2}                     & \multicolumn{2}{c}{Sparsity 0.1}                     \\ \midrule
		& \multicolumn{1}{c}{Train} & \multicolumn{1}{c}{Test} & \multicolumn{1}{c}{Train} & \multicolumn{1}{c}{Test} & \multicolumn{1}{c}{Train} & \multicolumn{1}{c}{Test} \\ \midrule
		RBM         &      0.1746                    &       0.1775                   &   0.1724                        &    0.1751                      &            0.1721               &      0.1746                    \\
		O-RBM       &    0.1754                       &     0.1782                     &       0.1737                    &         0.1764                 &      0.1710                     &   0.1735                       \\
		TI-RBM \cite{Sohn2012}      &      0.8434                     &      0.8295                    &  0.8106                          &         0.8118                 &                    0.7878       &          0.7917                \\
		$\boldmath{\theta}$\textbf{-RBM}   & \textbf{0.9000}                    & \textbf{0.9062}                   & \textbf{0.9087}                   & \textbf{0.9098}                  & \textbf{0.9103}                    & \textbf{0.9093}                   \\ \bottomrule
	\end{tabular}
\end{table}

\mypar{Training on \textit{mnist}:} We trained our network using the \textit{mnist} dataset, which does not contain rotated digit images. Since the \textit{mnist-rot} training set has 10,000 images, to achieve equal datasets and meaningful comparisons, we randomly sample the \textit{mnist} to have 10,000 training data as well. Keeping all the parameters the same and for sparsity $0.3$, experimental results showed that on the training set our $\theta$-RBM had a lower $\gamma$-score ($\sim 0.77$) compared to TI-RBM ($\sim 0.82$). However, TI-RBM had a much lower $\gamma$-score on the testing set $\sim 0.25$, compared with our method ($\sim 0.77$), highlighting the fact that TI-RBM tends to overfit across the support set of transformations.


\mypar{Effect of orientation estimation:} With this experiment, we want to demonstrate that even when the estimation of the dominant orientation is imperfect our system is still able to learn discriminative rotation-invariant features. To simulate this we randomly perturb the estimated orientation by adding an $\pm \epsilon$, which is an error term drawn from a certain probability distribution. Namely, since orientation is enumerated in $\mathcal{S}$ with an index $s$, that is if an image $x$ has orientation of $\varphi_s \in \mathcal{S}$, hence it has the $s$-th orientation, the $\epsilon$ affects this assignment, such that $s \leftarrow s \pm \epsilon$. In our experiments we wanted to have an $\epsilon$ of up to $n \in \mathbb{Z}$ with a certain probability $p$. Therefore we used a Bernoulli probability, such that $\epsilon \sim B(n,p)$, which indicates that $\epsilon$ can be up to $n$ with a probability of success $p$. \footnote{The sign of the error term can be drawn from an uniform distribution.} Experimental settings remained the same as previously with sparsity 0.3. In \Cref{tab:biased_est} we show the experimental results, in terms of testing error, using a RBF-SVM classifier. The SVM has trained using the same setup as reported above. Comparing these results with \Cref{tab:svm}, we see that even with large probability of error and even up to 4 errors (in index location)  the testing error is not affected (in Table 3 the $*$ denotes the values that are larger than the unperturbed value of $10.31$ in Table 1). 

\mypar{Implementation details:} We used our own implementation of Restricted Boltzmann machine, which we then used to build upon our $\theta$-RBM. Our code was written in MATLAB and runs on a machine with CUDA capabilities (NVIDIA Titan X). We download TI-RBM from the following URL: \url{https://github.com/kihyuks/icml2012_tirbm}, in order to evaluate their method with ours. Our $\theta$-RBM can be downloaded by running the following command: \textit{git clone https://bitbucket.org/stsaft/rbm.git}.


\section{Conclusions}
\label{sec:conclusions}

\begin{table}[t]
\centering
\caption{Testing error on \textit{mnist-rot} dataset. The training was done on biased estimation of the image orientation with an error drawn from a Bernoulli distribution \textit{B(n,p)}. We reported the testing error, using different values for $n$ and $p$, namely introducing more error on the estimation of the orientation. For explanation on the $*$ see text.}
\label{tab:biased_est}
\begin{tabular}{@{}rcccc@{}}
\toprule
\multicolumn{1}{c}{\textbf{}} & \textit{p = 0.1} & \multicolumn{1}{c}{\textit{p = 0.2}} & \multicolumn{1}{c}{\textit{p = 0.3}} & \multicolumn{1}{c}{\textit{p = 0.4}} \\ \midrule
\textit{n = 1}                & 8.9\%            & 9.1\%                                      &       9.0\%                               & 9.6\%                                     \\
\textit{n = 2}                & 8.9\%            & 9.4\%                                      &      9.3\%                                &  11.0\%*                                    \\
\textit{n = 3}                & 9.0\%            & 9.9\%                                     &          10.0\%                            &      13.4\%*                                \\
\textit{n = 4}                & 9.0\%            &  10.9\%*                                    &            11.3\%*                          & 16.1\%*                                     \\ \bottomrule

\end{tabular}
\end{table}

In this paper we proposed an unfactored gated restricted Boltzmann machine to learn rotation-invariant features, which were demonstrated to also have strong discriminative power. Differently than the state of the art, our $\theta$-RBM does not  need any dataset augmentation. Moreover, it does not even require transformed versions of input images, as it is done e.g. in \cite{Memisevic2013}. The network takes as input an image and its dominant orientation, which is computed by means of histogram of gradients. The third-order tensor, which connects the three layers in the network, is updated using an extended version of the Contrastive Divergence (CD) algorithm \cite{Hinton2002}. Since each image is associated to a particular slice in the tensor, determined by its dominant orientation, the gradient of the slice is computed using the original CD. Then, the remaining slices are generated by rotated versions of the computed gradient.

Using an SVM classifier, our method had the lowest classification error, compared with the current state of the art, as shown in \Cref{tab:svm}. Moreover, we mathematically proved that our method learns rotation invariant features, which was also demonstrated experimentally in \Cref{tab:inv_meas}, using the invariance measure proposed in \cite{Rasmus2015}, over the \textit{mnist-rot} dataset \cite{Larochelle2007}. This measure ranges in $[0,1]$, where `1' means that features are fully invariant to transformations (in our case, rotations). In fact, our proof shows that our learning algorithm can reach the highest value theoretically. Furthermore, experimentally our method had an invariance score of $\sim90\%$, which is the highest amongst baseline approaches and current state of the art approaches in transformation invariant RBM \cite{Sohn2012}. 


\section*{Acknowledgements}
We thank NVIDIA corporation for providing us a GPU Titan X video card.

\bibliographystyle{abbrvnat}
\bibliography{references}
\end{document}